\documentclass[journal]{IEEEtran}
\ifCLASSINFOpdf
\else
\fi
\usepackage{dblfloatfix}

\usepackage{amssymb}
\usepackage{amsmath}
\usepackage{amsthm}

\newtheoremstyle{mystyle}
  {}
  {}
  {\itshape}
  {}
  {\bfseries}
  {.}
  { }
  {}

\theoremstyle{mystyle}

\newtheorem{definition}{Definition}
\newtheorem{theorem}{Theorem}
\newtheorem{proposition}{Proposition}

\usepackage{booktabs}
\usepackage{multirow}
\usepackage{pifont}
\usepackage{tabularx,ragged2e,caption}
\usepackage{makecell}
\usepackage{tikz}
\usepackage{graphicx}
\usepackage{color}

\newcommand*\rott{\rotatebox{45}}
\newcommand*\rot{\rotatebox{90}}

\newcommand\diag[4]{%
  \multicolumn{2}{|p{#2}|}{\hskip-\tabcolsep
  $\vcenter{\begin{tikzpicture}[baseline=0,anchor=south west,inner sep=#1]
  \path[use as bounding box] (0,0) rectangle (#2+2\tabcolsep,\baselineskip);
  \node[minimum width={#2+2\tabcolsep},minimum height=\baselineskip+\extrarowheight] (box) {};
  \draw (box.north west) -- (box.south east);
  \node[anchor=south west] at (box.south west) {#3};
  \node[anchor=north east] at (box.north east) {#4};
 \end{tikzpicture}}$\hskip-\tabcolsep}}

\begin{document}
\title{Review \& Perspective for  \\Distance Based Trajectory Clustering \\}

\author{Philippe~Besse,~\IEEEmembership{Universit\'e de Toulouse INSA,~Institut de Math\'ematiques UMR CNRS 5219,}
        Brendan~Guillouet,~\IEEEmembership{Institut de Math\'ematiques UMR CNRS 5219,}
        Jean-Michel~Loubes,~\IEEEmembership{Universit\'e de Toulouse UT3,~Institut de Math\'ematiques UMR CNRS 5219,}
        ~and~Fran\c{c}ois~Royer,~\IEEEmembership{Datasio}
\thanks{}}

\markboth{}%
{}

\maketitle
\begin{abstract}
In this paper we tackle the issue of clustering trajectories of geolocalized observations.  Using clustering technics based on the choice of a distance between the observations, we  first provide a comprehensive review of the different distances used in the literature to  compare trajectories. Then based on the limitations of these methods, we  introduce a new distance : \textit{Symmetrized Segment-Path Distance (SSPD)}. We finally compare this new distance to the others  according to their corresponding clustering results obtained using both \textit{hierarchical clustering} and \textit{affinity propagation} methods.  \end{abstract}


%
\IEEEpeerreviewmaketitle

\section*{Introduction}

\IEEEPARstart{A}{trajectory} is a set of positional information for a moving object, ordered by time. This kind of multidimensional data is prevalent in many fields and applications, for example, to understand migration patterns by studying trajectories of animals, predict meteorology with hurricane data, improve athletes performance, etc. Our study is concentrated on vehicle trajectories within a road network. The growing use of GPS receivers and WIFI embedded mobile devices equipped with hardware for storing data enables us to collect a very large amount of data, that has to be analyzed in order to extract any relevant information. The complexity of the extracted data makes it a difficult challenge.
In this context, the goal of this work is to construct, in a data driven way, a collection of trajectories that will model the behaviors of car drivers. These models will be learned from a data set of time-stamped locations of cars. We focus in this work at the clustering of trajectories of vehicles. The natural application of this work is the forecast of the destination of drivers according to the shape of their trajectories. To achieve this goal, the first step is to cluster trajectories having similar paths. This clustering is based  on comparison between trajectory objects. This requires a new definition of distance between these objects which are studied.

A large amount of work has been done to give new definitions of trajectory distance. Tiakas \textit{et al.}(2009\cite{Tiakas2009}) , Rossi \textit{et al.} (2012\cite{Rossi2012}), Han \textit{et al.}(2012\cite{Han2012}) or Hwang \textit{et al.}(2005\cite{Hwang2005}) propose road network based distances. They assume that the trajectories studied are perfectly mapped on the road network. However, this task is strongly dependent on the precision of the GPS device. When the time interval between two GPS locations is significant, several paths on the graph are possible between locations, especially when the network is dense. Moreover it requires the knowledge of the road network. Here, we focus on completely data driven methods without	 any a priori information.
Several methods have been used to cluster data set of trajectories. Clustering methods using the Euclidean distance lead to bad results mainly due to the fact that trajectories have different lengths.
Hence, several methods based on warping distance have been defined , Berndt (1994\cite{Berndt1994}), Vlachos \textit{et al.} (2002\cite{Vlachos2002}), Chen \textit{et al.} (2004\cite{Chen2004}), and Chen \textit{et al.} (2005 \cite{Chen2005}).  These methods reorganize the time index of trajectories to obtain a perfect match between them. Another point of view is to focus on the geometry of the trajectories, in particular on their shape. Shape distances like Hausdorff and Fr\'echet distances can be adapted to trajectories but fail to compare them as a whole. Lin \textit{et al.} (2005\cite{Lin2005}) proposed a method based exclusively on the shape of the trajectory but at high computational cost.

In section \ref{section_review} of this paper several distances are studied and compared. A new distance will be presented in section \ref{section_new_distance}: the Symetrized Segment-Path Distance (\textit{SSPD}). \textit{SSPD} is a shape-based distance that does not take into account the time index of the trajectory. It compares trajectories as a whole, and is less affected by incidental variation between trajectories. It also takes into account the total length, the variation and the physical distance between two trajectories. To evaluate our distances, and compare them to others, clustering results of some trajectory sets are analyzed in section \ref{section_results}.

\section{Model for trajectory clustering}\label{model_trajectory_clustering}

\subsection{Trajectory}
A continuous trajectory is a function which gives the location of a moving object as a continuous function of time. In our case we will only consider discrete trajectories defined here after.

\begin{definition} A trajectory $T$ is defined as  \\
$T$ : $((p_1,t_1),\hdots,(p_{n},t_{n}))$, \\ where $p_k  \in \mathbb{R}^2,t_k \in \mathbb{R} ~ \forall k \in [1\hdots n], ~\forall n \in \mathbb{N} $ and $n$ is the length of the trajectory $T$.
\end{definition}

The exact locations between time $t_i$ and $t_{i+1}$ are unknown. When these locations are required, a piece wise linear representation is used between each successive location $p_i$ and $p_{i+1}$ resulting in a line segment $s_i$ between these two points. This new representation is called a piece wise linear trajectory. In this representation, no assumption is made about time indexing of segment $s_i$.

\begin{definition} A piece wise linear trajectory is defined as 
$T_{pl}$ : $((s_1),\hdots,(s_{n-1}))$ , where $s_k \in \mathbb{R}^2$ and $n_{pl}$ is the length of the trajectory.
\end{definition}
The length of the trajectory $n_{pl}$ is the sum of the lengths of all segments that compose it : $n_{pl} = \sum_{i \in [1 \hdots n-1]} \|p_ip_{i+1}\|_2$.

The notation used in this paper are summarized in Table \ref{table_notation}.

\begin{table}[!t]
\centering
\caption{Notation}\label{table_notation}
\begin{tabular}{|c|l|} \hline
 $\mathcal{T}$ & The set of trajectories \\ \hline
 $T^i$ & The $i^{th}$ trajectory of set $\mathcal{T}$ \\ \hline
 $T^i_{pl}$ & The piece wise linear representation of $T^i$ \\ \hline 
 $n^i$ & Length of trajectory $T^i$ \\ \hline
 $n^i_pl$ & Length of the $T^i_pl$ \\ \hline
 $p^i_k$ & The $k^{th}$ location of $T^i$\\ \hline
 $p^i$ & The set of continuous points that compose $T^i_{pl}$  \\ \hline
 $s^i_k$ & The line segment between  $p^i_j$ and  $p^i_{k+1}$  \\ \hline
 $t^i_k$ & The time index of location $p^i_k$\\ \hline
 $\|p_kp_l\|_2$ & The Euclidean distance between $p_k$ and $p_l$ \\ \hline 
\end{tabular}
\end{table}

\subsection{Distance}
There are many ways to define how close two objects are far one from another. Beyond the notion of mathematical distance, many functions can be used to qualify this dissimilarity. The terminology used in literature to define them is not completely standardized. Therefore we will use the definition established in Deza \textit{et al.} (2009\cite{Deza2009}) as a reference.

\begin{definition}Let $\mathcal{T}$ be a set of trajectories. A function d : $\mathcal{T} \times \mathcal{T} \mapsto \mathcal{R}$ is called a $dissimilarity$ on $\mathcal{T}$ if for all $T^1,T^2 \in \mathcal{T}:$ 
\begin{itemize}
\item $d(T^1,T^2) \geq 0$
\item $d(T^1,T^2) = d(T^2,T^1)$
\item $d(T^1,T^1)=0$
\end{itemize}
If all of these conditions are satisfied and $d(T^1,T^2)=0 \implies T^1=T^2$ d is considered to be a $symmetric$. If the triangle inequality is also satisfied, d is called a $metric$. These notations are summarized in Table \ref{metrics}.
\end{definition}

\begin{table*}[!t]
\centering
\caption{Metric Definition}\label{metrics}
\begin{tabular}{|ll|ccc|}
\hline
\diag{2em}{30em}{ Property }{ Metric Name } &
\rott{\textit{dissimilarity}} &
\rott{\textit{symmetric}} &
\rott{\textit{metric}}
    \\ \hline
Non-Negativity & $D(T^1,T^2) \geq 0$                          & X  & X  &$\ast$\\ 
Symmetry & $D(T^1,T^2) = D(T^2,T^1)$                          & X  & X  & X  \\
Reflexivity & $D(T^1,T^1)=0 $                                 & X  &$\ast$& $\ast$\\
Triangle Inequality & $D(T^1,T^3) \leq D(T^1,T^2)+D(T^2,T^3)$ &    &    & X  \\
Identity of indiscernible & $D(T^1,T^2)=0 \implies T^1=T^2 $  &    & X  & X  \\ \hline
\end{tabular}
\end{table*}
X indicates the required properties for each distances, while $\ast$ indicates properties that are automatically satisfied (by the presence of the other required properties for the metric). 

\subsection{Desired properties of clustering and distances}\label{desired_properties}
Our aim is to be able to predict the most probable next location of a moving object given few location data points. This prediction should be based on groups of past trajectories that have been gathered together sharing a similar behavior. Hence we aim at finding a clustering method that should regroup trajectories 
\begin{itemize}
\item with similar shape and length 
\item which are physically close to each other
\item which are similar as a whole with more than just similar sub-parts
\item all of these properties should be considered without regard to their time indexing
\end{itemize}

Moreover we want to design a very general procedure able to treat all trajectories data, without a prior knowledge on the particular geographical location where they are collected.
To obtain such clustering, the issue of this work is to find a distance that respects such properties and succeed in extracting these features. Actually, the desired distance should have the following properties,

\begin{itemize}
\item it compares distances as a whole
\item the compared trajectories can be of different lengths,
\item the time indexing can be very different from one trajectory to another
\item the trajectories can have similar shapes but can be physically far from each other and vice versa
\item extra parameters should not be required.
\end{itemize} 

\section{Distance on trajectories: a review} 
\label{section_review}

Three main kind of distances have been introduced in the literature. The first uses the underlying road network, \textbf{Network-Constrained Distance}. Theses distances will not be detailed in this paper.  They assume that the road network is known and that trajectory data are perfectly map on it.  Distances that do not use the underlying road network can also be classified into two categories: those who only compare the shape of the trajectory, \textbf{Shape-Based Distance} and those who take into account the temporal dimension; \textbf{Warping based Distance}. 

Performance of clustering algorithms using these distances will be compared section \ref{model_trajectory_clustering}, as well as their computation cost and their metric properties.

\subsection{Warping based Distance} 

Euclidean distance, Manhattan distance or other $L^p$-norm distances are the most obvious and the most often used distances. They compare discrete objects of the same length. They can be used to look for common sub-trajectories of a given length but they can not be used to compare entire trajectories. Moreover, these distances will compare locations with common indexes one by one.  At a given index $i$, location $p^1_i$ of trajectory $T^1$  will be compared only to location $p^2_i$ of trajectory $T^2$. However, these locations  can be strongly different according to the speeds of the trajectories. Hence, it makes no sense to compare them without taking this into account. This problem is also common in time series analysis and not only in trajectory analysis.   

Warping distance aims to solve this problem. For this purpose, they enable to match locations from different trajectories with different indexes. Then, they find an optimal alignment between two trajectories, according to a given cost $\delta$ between matched location.
Several warping based distance have been defined. DTW ({Berndt \textit{et al.}, (1994 \cite{Berndt1994})) and later $LCSS$ (Vlachos \textit{et al.}, 2002\cite{Vlachos2002}), $EDR$ (Chen \textit{et al.}, 2005\cite{Chen2005}) and $ERP$ (Chen \textit{et al.}, 2005\cite{Chen2004}).  These distances are defined the same way, but they use different cost functions.

In order to define a warping distance, two compared time series trajectories, $T^i, T^j$, are arranged to form a $n^i\times n^j$ grid $G$. The grid cell, $g_{k,l}$, corresponds to the pair ($p^i_k$,$p^j_l$).
\begin{definition}
A warping path, $W=w_1,\hdots,w_{|W|}$, crosses the grid $G$ such that
\begin{itemize}
\item $w_1 = g_{1,1}$,
\item $w_{|W|}=g_{n^i,n^j}$,
\item if $w_k=g_{k_i,k_j}$, then $w_{k+1}$ is equal to $ g_{k_i+1,k_j}$, $g_{k_i,k_j+1}$ or $g_{k_i+1,k_j+1}$.
\end{itemize}
\end{definition}

The order of the locations in a trajectory are maintained but they can be repeated, deleted or replaced by an arbitrary value, a $gap$, along the warping path.
The distance is then computed by minimizing or maximizing the sum of a given cost $\delta$ between all pair of locations that make a warping path $W$, among all existing warping path.

\begin{definition}\label{def_warping} A warping distance is defined as 
\begin{equation}
\begin{array}{rl}
D(T^i,T^j) = & \min_W  \Big[ \sum_{k=1}^{|W|} \delta(w_k) \Big], \\
\text{ or } = & \max_W  \Big[ \sum_{k=1}^{|W|} \delta(w_k) \Big],
\end{array}
\end{equation}

where 
$\delta(w_k)=\delta(g_{k_i,k_j})=\delta(p^i_{k_i},p^j_{k_j}),$ is the cost function and $W$ is a warping path.

\end{definition}

They are generally computed by dynamic programming. Table \ref{re_indexing_based_distance_definition} displays the cost functions as well as the dynamic formulation of these distances.

\begin{table*}[!t]
\centering
\caption{Re-Indexing based distance definition}\label{re_indexing_based_distance_definition}

\resizebox{0.8\linewidth}{!}{%
\makebox[\textwidth][c]{
\begin{tabular}{c|c|c|c l|}\cline{2-5}

& & Cost function & \multicolumn{2}{c|}{Distance}  \\
& & $\delta_{NAME}(p_1,p_2)$ = & \multicolumn{2}{c|}{$NAME(T^i,T^j)$ =}    \\ \cline{2-5}
 & \rot{DTW}  &$\|p_1p_2\|_2$ & 
$=$ &  $\left\{
  \begin{array}{l l }
    0 & \text{if } n^i=n^j = 0  \\
    \infty & \text{if } n^i=0\text{ or }n^j = 0 \\
     \begin{array}{l }
     \delta_{DTW}(p^i_1,p^j_1) + \\
     \min \bigg\{ \begin{array}{l }
     DTW(rest(T^i),rest(T^j)), \\
     DTW(rest(T^i),T^j)),\\
     DTW(T^i,rest(T^j)
     \end{array}\Bigg\}
     \end{array}  & \text{otherwise} 
  \end{array} \right.$
 \\
 \cline{2-5}

\multirow{4}{*}{$\rot{NAME}$} & \rot{LCSS}  & $\left\{
  \begin{array}{l l}
    1 & \text{if } \|p_1p_2\|_2 < \varepsilon_d  \\
    0 &  \text{if } p_1\text{ or }p_2\text{ is a gap} \\
    0 & \text{otherwise}
  \end{array} \right.$ &
$=$ & $\left\{
  \begin{array}{l l }
    0 & \text{if } n^i=0\text{ or }n^j = 0  \\
    LCSS(rest(T^i),rest(T^j)) + \delta_{LCSS}(p^i_1,p^j_1) & \text{if } \delta_{LCSS}(p^i_1,p^j_1) = 1 \\
     \begin{array}{l }
     \max \bigg\{ \begin{array}{l l l }
     LCSS(rest(T^i),T^j)) & + & \delta_{LCSS}(p^i_1,gap),\\
     LCSS(T^i,rest(T^j)) & + & \delta_{LCSS}(gap,p^j_1)
     \end{array}\Bigg\}
     \end{array}  & \text{otherwise} 
  \end{array} \right.$ \\
\cline{2-5}

& \rot{EDR} & $\left\{
  \begin{array}{l l}
    0 & \text{if } \|p_1p_2\|_2 < \varepsilon_d  \\
    1 &  \text{if } p_1\text{ or }p_2\text{ is a gap} \\
    1 & \text{otherwise}
  \end{array} \right.$ & 
$=$ & $\left\{
  \begin{array}{l l }
    n^i & \text{if } n^j = 0  \\
    n^j & \text{if } n^i = 0  \\
    EDR(rest(T^i),rest(T^j)) & \text{if }  \delta_{EDR}(p^i_1,p^j_1)=0 \\
     \begin{array}{l }
     \min \bigg\{ \begin{array}{l l l }
     EDR(rest(T^i),rest(T^j)) & + & \delta_{EDR}(p^i_1,p^j_1), \\
     EDR(rest(T^i),T^j)) & + & \delta_{EDR}(p^i_1,gap) ,\\
     EDR(T^i,rest(T^j) & + & \delta_{EDR}(gap,p^j_1)
     \end{array}\Bigg\}
     \end{array}  & \text{otherwise} 
  \end{array} \right.$ \\
  \cline{2-5}

& \rot{ERP} &  $\left\{
  \begin{array}{l l}
    \|p_1p_2\|_2 &\text{if } p_1\text{, }p_2\text{ are not gaps} \\
    \|p_1g\|_2 &  \text{if } p_2\text{ is a gap} \\
    \|gp_2\|_2 & \text{if } p_1\text{ is a gap} \\
  \end{array} \right.$ & 
$=$ & $\left\{
  \begin{array}{l l }
    \sum_{k=1}^{n^i} \|p^i_kg\|_2 & \text{if } n^j = 0  \\
    \sum_{l=1}^{n^j} \|p^j_lg\|_2 & \text{if } n^i = 0  \\
     \begin{array}{l }
     \min \bigg\{ \begin{array}{l l l }
     ERP(rest(T^i),rest(T^j)) & + & \delta_{ERP}(p^i_1,p^j_1), \\
     ERP(rest(T^i),T^j)) & + & \delta_{ERP}(p^i_1,gap) ,\\
     ERP(T^i,rest(T^j) & + & \delta_{ERP}(gap,p^j_1)
     \end{array}\Bigg\}
     \end{array}  & \text{otherwise} 
  \end{array} \right.$ \\
  \cline{2-5}

\end{tabular}}%
}

\end{table*}

On contrary to the three other distances, $LCSS$ is a similarity. The exact similarity used in Vlachos \textit{et al.}, 2002\cite{Vlachos2002} is $S(T^i,T^j)=\frac{LCSS(T^i,T^j)}{\min\{n^i,n^j\}}$, which is between $0$ and $1$. We will then use the distance
$$DLCSS(T^i,T^j) = 1-S(T^i,T^j),$$ to compare distances to each other.

The metric types of these distance functions, and computational cost for the four methods are summarized in table \ref{re_indexing_based_distance_properties}. \\

\subsubsection{Comparisons}
\begin{itemize}
\item All of these distances handle local time shifting.
\item The cost function $\delta$ uses the Euclidean distance. Some of this distances have been defined using a L1-norm, but Euclidean distance is more adapted for real values.
\item $LCSS$ and $EDR$'s cost function count the number of occurence where the Euclidean distance between matched location does not match a spatial threshold, $\varepsilon_d$.  The former counts similar locations, the latter the difference. This threshold makes the distance robust to noise. However, it has a strong influence on the final results. If the threshold is large, all the distances will be considered similar and if low, only those having very close locations will be considered similar.  
\item In comparison, $ERP$ and $DTW$ put a weight to these differences by computing the real distance between the locations. In this sense they can be viewed as more accurate.
\item $ERP$ is the only distance which is a $metric$ regardless of the $L_p$ norm used, yet it works better for normalized sequences, especially for defining the gap value $g$. It does not apply for vehicle trajectories.
\item In addition, these distances may include a time threshold, $\varepsilon_t$. Thus, two locations will not be compared if the difference between their time indexing is too large. However, it is very hard to estimate the value of this threshold when comparing trajectories due to the presence of noise.\end{itemize}

\begin{table}[!t]
\centering
\caption{Re-Indexing based distance properties}\label{re_indexing_based_distance_properties}

\begin{tabular}{|c|c|c|} \hline
Name & Metric Types  & Computation \\
&   & Cost \\ \hline
DTW  & \textit{symmetric} &  $O(n^2)$ \\ \hline
LCSS  & \textit{distance} & $O(n^2)$ \\ \hline
EDR  & \textit{symmetric} & $O(n^2)$ \\ \hline
ERP  & \textit{metric} & $O(n^2)$ \\ \hline
\end{tabular}

\end{table}

\subsubsection{Pros and Cons}

The main advantage of these distances is that they enable comparison of sequences of different lengths.

The two main limitations of warping based distance are the following 
\begin{itemize}

\item Warping methods are based on one-to-one comparison between sequences. Hence, it often requires the choice of a particular series that will be used as a reference, onto which all other sequences will be matched. The indexed of two sequences that are compared should be well balanced in order to capture best the variability. For instance to detect if there were accelerations and decelerations during the measurement of the time series. Hence the choice of the reference sequence is very important.
\item The performance of usual methods based on warping techniques is hampered by the large amount of noise inherent to road traffic data, which is not the case when studying time series.
\end{itemize}
Instead of correcting the time index, the solution is to use distances that have the effect of time removed.

\subsection{Shape-Based Distance}

These distances try to catch geometric features of the trajectories, in particular, their shape. Among \textbf{Shape-Based Distances}, the Hausdorff distance (Hausdorff, 1914 \cite{Hausdorff1914}), and the Fr\'echet distance (Fr\'echet, 1906\cite{Frechet1906}) are likely the most well known.\\

\subsubsection{Hausdorff}

The \textit{Hausdorff} distance is a \textit{metric}. It measures the distance between two sets of metric spaces. Informally, for every point of set 1, the infimum distance from this point to any other point in set 2 is computed. The supremum of all these distances defines the \textit{Hausdorff} distance.

\begin{definition} The \textit{Hausdorff} distance between two sets of metric spaces is defined as
$$Haus(X,Y) = \max\{\displaystyle \sup_{x\in X} \displaystyle \inf_{y\in Y} \|xy\|_2,\displaystyle \sup_{y\in Y} \displaystyle \inf_{x\in X} \|xy\|_2 \}.$$
\end{definition}

This distance is complicated and resource intensive to calculate when applied to most existing sets. But in the case of polygonal curves like trajectories, some simplification can be made due to the monotonic properties of a segment. Distance from a point $p$ to a segment $s$ is defined as follows.

\begin{definition} $Point-to-Segment$ distance.\label{definition-distance-segment-trajectory}
$$ \small \displaystyle
D_{ps}(p^1_{i_1},s^2_{i_2})=\left\{
   \begin{array}{ll}
       \|p^1_{i_1}p^{1proj}_{i_1}\|_2 & \mbox{if } p^{1proj}_{i_1} \in s^2_{i_2}, \\
       \min(\|p^1_{i_1}p^2_{i_2}\|_2,\|p^1_{i_1}p^2_{i_2+1}\|_2) & \mbox{otherwise.}
   \end{array}
\right.$$
\end{definition}
\textit{Where $p^{1proj}_{i_1}$ is the orthogonal projection of $p^{1}_{i_1}$ on the segment $s^2_{i_2}$}.\\

Hence, the \textit{Hausdorff} distance between two line segments is
$$\arraycolsep=1.4pt
\begin{array}{lll}
D_{Haussdorf}(s^1_{i_1},s^2_{i_2}) & = \max\{ &  \sup_{p\in s^1_{i_1}} D_{ps}(p,s^2_{i_2}), \\
& & \sup_{p\in s^2_{i_2}} D_{ps}(p,s^1_{i_1})\} \\
& =\max\{ & D_{ps}(p^1_{i_1},s^2_{i_2}),D_{ps}(p^1_{i_1+1},s^2_{i_2}),\\
& & D_{ps}(p^2_{i_2},s^1_{i_1}),D_{ps}(p^2_{i_2+1},s^1_{i_1}) \}. \\
\end{array}
$$

\begin{figure}[!ht]
\centering
\includegraphics[width=\linewidth]{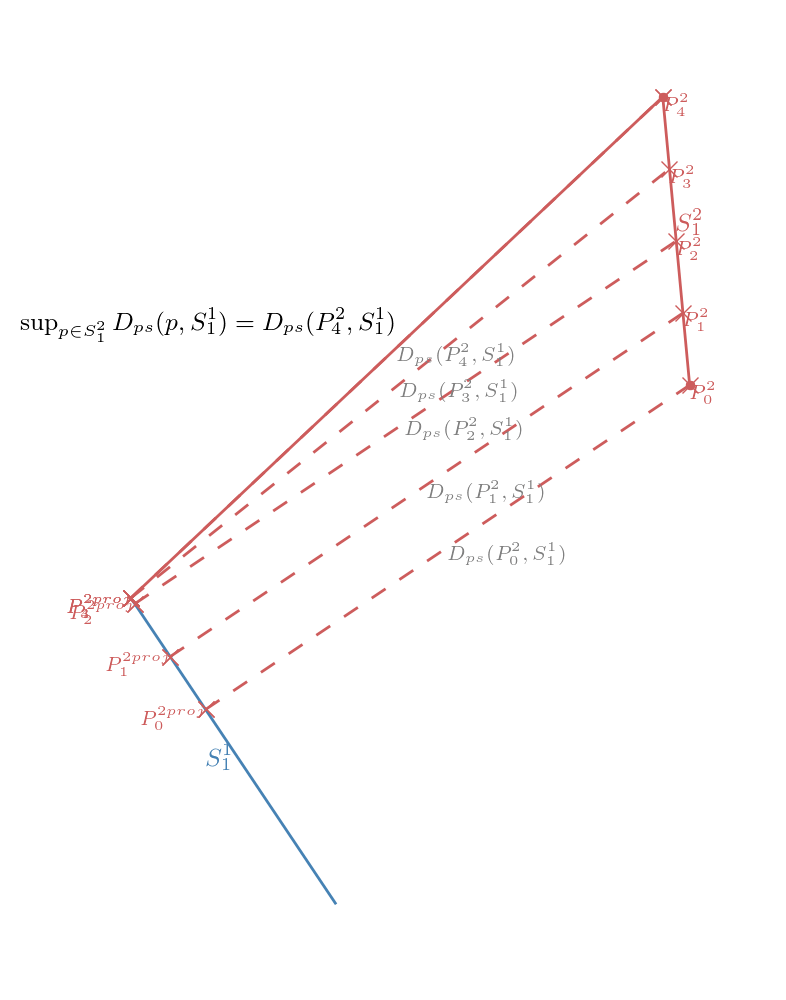}
\caption{Supremum of $Point-to-Segment$ distance from point of segment $s^1_1$ to segment $s^2_1$}
\label{hausdorff_segment}
\end{figure}

Indeed, a segment is monotonic. As seen in Fig. \ref{hausdorff_segment}, the supremum of the $Point-to-Segments$ distance from any points of a segment $s^1_{i_1}$ to a segment $s^2_{i_2}$ occurs at one of the end points of the segment $s^1_{i_1}$. The \textit{Hausdorff} distance between two trajectories can then be computed with the following formula.

\begin{definition} \textit{Hausdorff} distance between two discrete trajectories.
$$\small \begin{array}{ll}
D_{Haussdorf}(T^1,T^2)=\max\big\{ &  \max_{\substack{i_1 \in [1\hdots n^1] \\
j_2\in [1\hdots n^2-1]} } \{D_{ps}(p^1_{i1},s^2_{j2}\}, \\
&  \max_{\substack{j_1 \in [1\hdots n^1-1] \\
i_2\in [1\hdots n^2]} } \{D_{ps}(p^2_{i2},s^1_{j1}\} \big\}.
\end{array}$$
\end{definition}

The \textit{Hausdorff} distance can then be computed in a $O(n^2)$ computational time.

\subsubsection{Frechet and discrete Fr\'echet }
The \textit{Frechet} distance measures similarity between curves. It is often known as the "walking-dog distance". Imagine a dog and its owner walking on two separate paths without backtracking from one endpoint to one other. The Fr\'echet distance is the minimum length of leash required to connect a dog and its owner. While the Hausdorff distance takes distance between arbitrary point, the Fr\'echet metric takes the flow of the two curves into account.

\begin{definition} The \textit{Frechet} distance between two curves is defined as
$$D_{Frechet}(A,B) = \displaystyle \inf_{\alpha,\beta \in X} \displaystyle \max_{t \in [0,1]} \Big\{ \|A\big(\alpha(t),B\big(\beta(t)\big)\|_2 \Big\}.$$
\end{definition}

As well as the \textit{Hausdorff} distance, the \textit{Frechet} distance is a \textit{metric}. It is also resource intensive. Alt \textit{et al.} (1995\cite{Alt1995}) developed an algorithm measuring the exact Fr\'echet distance for polygonal curves based on the free space definition.

\begin{definition} A free space $F_{\epsilon}(T^1,T^2)$ between two trajectories is the set of all pairs of points whose distance is at most $\epsilon$.
$$ F_{\epsilon}(T^1,T^2):=\{(p^1,p^2) \in (T^1,T^2)\} | \|p^1,p^2\|_2 \leq \epsilon \}.$$
\end{definition}

The \textit{Fr\'echet distance} between two trajectories $T^1$ and $T^2$ is the minimum value of $\epsilon$ for which a curve exists within the corresponding $F_{\epsilon}$ from $(p^1_0,p^2_0)$ to $(p^1_{n^1},p^2_{n^2})$ with the property of being monotone in both trajectories. Computing the \textit{Fr\'echet distance}  means finding this minimum value of $\epsilon$. By exploiting the monotonic property of the segments and the definition of the free space, this task can be accomplished more efficiently.

Indeed, the \textit{Frechet} distance between segments is equal to the \textit{Hausdorff} distance between segments, \textit{i.e.} 

$$
\begin{array}{rll}
D_{Frechet}(s^1_{i_1},s^2_{i_2})&=\max\{ &  D_{ps}(p^1_{i_1},s^2_{i_2}), \\
& & D_{ps}(p^1_{i_1+1},s^2_{i_2}),\\
& & D_{ps}(p^2_{i_2},s^1_{i_1}), \\ 
& &D_{ps}(p^2_{i_2+1},s^1_{i_1}) \} \\
&=\epsilon_{i_1,i_2}.
\end{array}
$$

To compute the \textit{Frechet} distance between trajectories $T^1$ and $T^2$ , we only look among the set $E$ of  \textit{Frechet} distances between all pairs of segments of $T^1$ and $T^2$.  $E = \{\epsilon_{i_1,i_2}\text{ for } (i_1,i_2)\in([1\hdots n^1-1]\times[1\hdots n^2-1])\}$. This simplification enables us to compute the \textit{Frechet} distance between trajectories $T^1$ and $T^2$ in $O(n^2log(n^2))$. We highlight that this computational cost is higher than all the other distance studied. 

Eiter \textit{et al.} (1994\cite{Eiter1994}) describes an approximation of this distance for polygonal curves called the \textit{discrete Fr\'echet} distance. This distance is close to the definition of the warping based distance.
\begin{definition} The \textit{discrete Fr\'echet} distance is defined as
$$D_{Frechet-Discr}((T^1,T^2) = \min_W \{\displaystyle \max_{k \in [1 \hdots |W|]} \|w_k\|_2 \}.$$
\end{definition}
with $W$, the warping path defined in definition \ref{def_warping}. The \textit{discrete Fr\'echet} distance can be computed in  $O(n^2)$ time.

This distance is bounded as follows.

\begin{theorem}[]{For any trajectories $T^i$ and $T^j$ \cite{Eiter1994}}
$$\small
D_{Frechet}(T^i,T^j) \leq D_{Frechet-Discr}((T^i,T^j) \leq D_{Frechet}(T^i,T^j) + \epsilon $$ \\
Where, $\epsilon=\max\{\displaystyle \max_{k \in [1\hdots n^i-1]} \{ \|p^i_kp^i_{k+1}\|_2 \},\max_{l \in [1\hdots n^j-1]} \{ \|p^j_lp^j_{l+1}\|_2 \}\}.$
\end{theorem}

\subsubsection{One Way Distance}
Lin \textit{et al.} 2005\cite{Lin2005} defines the \textit{One-Way-Distance}, \textit{OWD}, from a trajectory $T^i$ to another trajectory $T^j$. It is defined as the integral of the distance from points of $T^i_{pl}$ to trajectory $T^j_{pl}$ divided by the length of $T^i_{pl}$

\begin{definition} The \textit{OWD} distance is defined as 
$$D_{OWD}(T^i,T^j) = \frac{1}{n^i_{pl}} \int_{p^i \in T^i_{pl}} D_{point}(p^iT^j) dp^i,$$
\end{definition}
where $D_{point}(p,T)$ is the distance from the point $p$ to the trajectory $T$ so that  

$$D_{point}(p,T) = \displaystyle \min_{q\in T_{pl}} \|pq\|_2.$$

The \textit{OWD} distance is not symmetric, but $D_{SOWD}(T^i,T^j) = (D_{OWD}(T^i,T^j) +D_{OWD}(T^j,T^i) )/2$ is. This distance is a $symmetric$ because it does not satisfy the triangle inequality.

Lin \textit{et al.}\cite{Lin2005},  have defined two algorithms to compute the \textit{OWD} in case of piecewise linear trajectories. 
\begin{itemize}
\item The first consists in finding the parametrized OWD function $D_{OWD}(s_k^i,T^j)$ from a segment $s_k^i$ of $T^i_{pl}$ to all segments $s^j$ of $T^j_{pl}$ and for all segments of $T_{pl}^i$
$$D_{OWD}(T^i,T^j) = \frac{1}{n^i_{pl}}\displaystyle \sum_{k=1}^{n^i-1}D_{OWD}(s^i_k,T^j).\|p^i_kp^i_{k+1}\|,$$
with a $O(n^2log(n))$ complexity.

\item The second one uses a grid representation of the trajectory. The space is discrete as we see in Fig. \ref{grid_rep}. Trajectory are defined as the succession of grids they crossed 

\begin{definition} A grid representation trajectory is defined as 
$$T_{grid} := (g_0,\hdots,g_{n_{grid}}),$$ where $g_n$ are cells of the discrete space.
\end{definition}

\begin{figure}[!ht]
\centering
\includegraphics[width=0.7\linewidth]{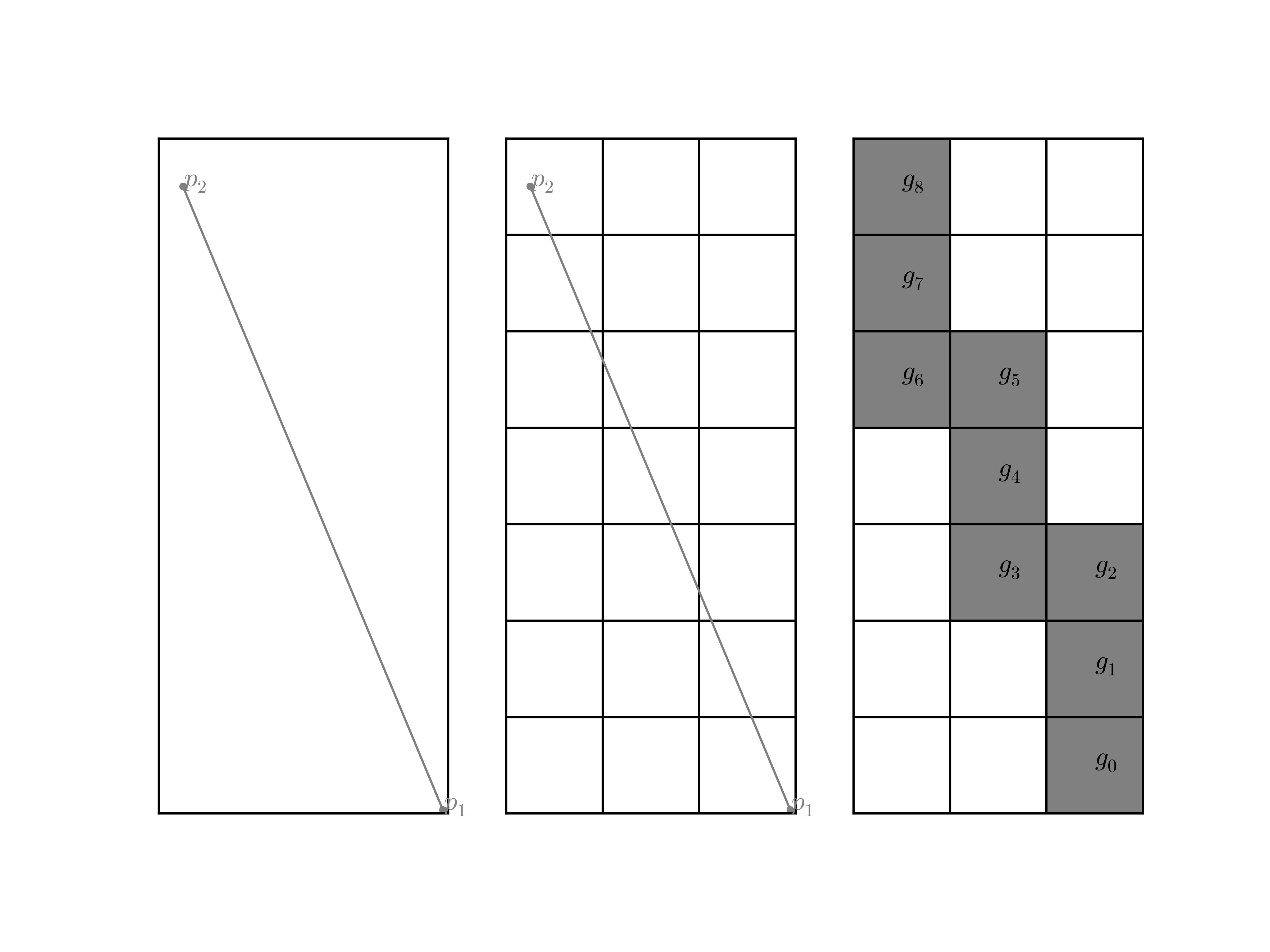}
\caption{Grid representation of a segment}
\label{grid_rep}
\end{figure}

This representation simplifies the computation and reduces the complexity to $O(nm)$ where $m$ is the number of local min points. Local min points of a grid cell $g$ are the grids with distances to $g$ shorter than those of their neighbors' grid cell.
\end{itemize}

Table \ref{shape_based_distance_properties} displays the metric types and the computational cost of these distances. 

 \begin{table}[!t]
\centering
\caption{Shape based distance properties}\label{shape_based_distance_properties}
\begin{tabular}{|c|c|c|} \hline
Name  & Metric Types  & Computation \\
&  & Cost \\ \hline
\textit{Hausdorff}  & \textit{metric} &  $O(n^2)$ \\ \hline
\textit{Frechet}   & \textit{metric} & $O(n^2log(n^2))$ \\ \hline
\textit{discrete Fr\'echet}   & \textit{symmetric} & $O(n^2)$ \\ \hline
\textit{OWD} & \textit{symmetric} & $O(n^2log(n))$ \\ \hline
\textit{$OWD_{grid}$}    & \textit{symmetric} & $O(mn)$ \\ \hline
\end{tabular}
\end{table}

\subsubsection{Pros and Cons}

\begin{itemize}
\item \textit{Frechet} and \textit{Hausdorff} distances are both \textit{metrics}, meaning they satisfy triangular inequality. With clustering algorithms like \textit{dbscan} or \textit{K-medoid} this is a necessary property for the distance used if we want the clustering algorithm to be efficient. They have been widely used in many domains where shape comparison is needed. But they can fail to compare trajectories as a whole. Indeed both \textit{Fr\'echet} and \textit{Hausdorff} distance return a maximum distance between two objects at given points of the two objects. As we can see in Fig. \ref{hausdorff_frechet_inconvenient}, despite the fact that the trajectories $T^1$ and $T^2$ are well separated at the maximum value of x, they are clearly more similar to each other than to $T^3$. But with Hausdorff calculated distance, there are no strong differences between $D_{Haussdorf}(T^1,T^2)$, $D_{Haussdorf}(T^1,T^3)$ and $D_{Haussdorf}(T^2,T^3)$. With Frechet, $D_{Frechet}(T^1,T^2)$ is even bigger than $D_{Frechet}(T^1,T^3)$ and $D_{Frechet}(T^2,T^3)$.

\begin{figure}[!ht]
\centering
\includegraphics[width=\linewidth]{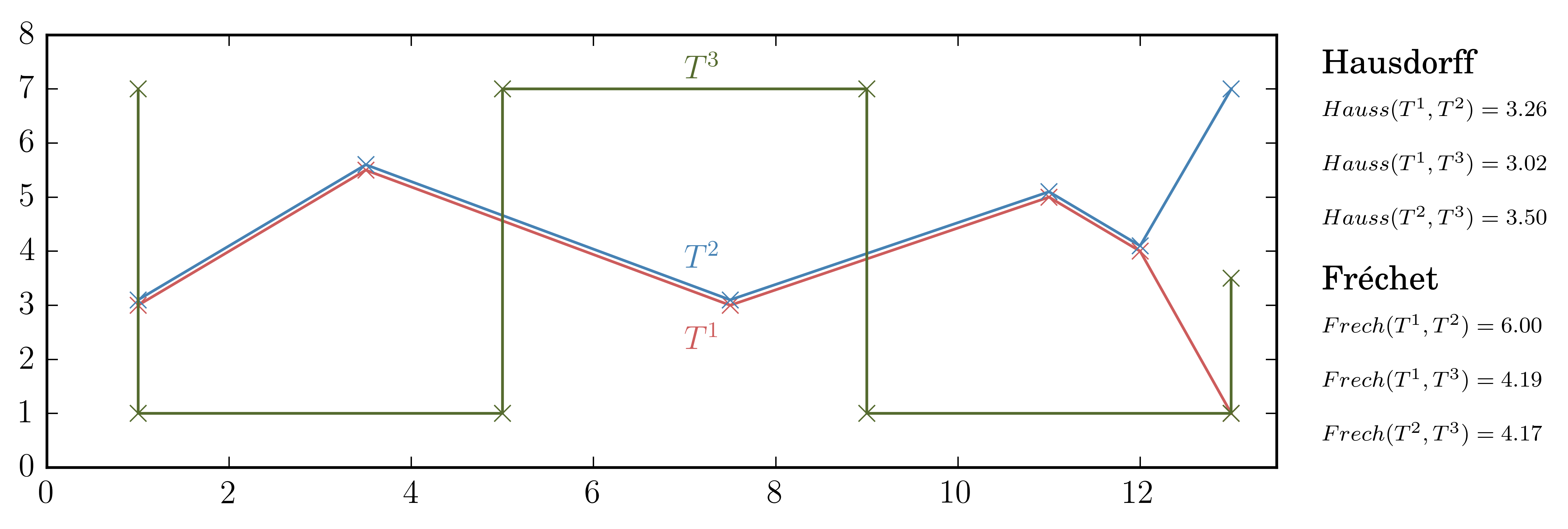}
\caption{Frechet And Hausdorff Computation between three trajectories}
\label{hausdorff_frechet_inconvenient}
\end{figure}

\item The \textit{Discrete Fr\'echet} distance requires considerably less computing time compared to the \textit{Frechet} distance. But \textit{Discrete Frechet} is not a \textit{metric}. Moreover, due to its similarity with the warping distance it inherits the same inconveniences.

\item  The distance present in Lin \textit{et al.} (2005\cite{Lin2005}) is by far the one that best meets our requirements. It compares trajectories as a whole, taking into account their shapes and their physical distances, the required features for our distance. However, its complexity makes it computationally slow. The algorithm for grid representation is faster. Its computational time is $O(mn)$. Yet it does not take into account the computation time required for matching the trajectory to the grid. Moreover, the size of the grid chosen strongly influences the final result and makes it imprecise.
Moreover, the distance gives the same "weight" to all points defining the trajectory: points directly issued from the GPS location, and points which compose the piece wise linear representation. The greater the length of the segment $s$ is, the stronger its influences on the trajectory is. The more separated the endpoints of a segment $s$ is, the less confident the interpolation between them is.  
\end{itemize}

In the following section, a new distance will be established inspired from both the OWD and the Hausdorff distances. 

\section{A new distance : Symmetrized Segment-Path Distance (SSPD)}
\label{section_new_distance}

The shape based distances are by far the distances that best fit the desired properties defined in section \ref{desired_properties}. However none of them matches it perfectly. Hence, a new distance that fits these requirements is provided in this section. The \textit{Symmetrized Segment-Path Distance}. This distance is a shape based distance.It takes into account the whole trajectories, and is less affected by noise.

The distance $D_{pt}$ from a point $p$ to a trajectory $T$ is the minimum of distances between this point and all  segments $s$ that compose $T$. The \textit{Segment-Path distance} from trajectory $T^1$ to trajectory $T^2$ is the mean of all distances from points composing $T^1$ to the trajectory $T^2$

\begin{definition}$SPD$ distance is defined as\label{definition-distance}
$$ 
D_{SPD}(T^1,T^2)=\frac{1}{n_1}\sum_{i_1=1}^{n_1} D_{pt}(p_{i_1}^1,T^2).
$$
where, $D_{pt}(p^1_{i_1},T^2)=\min_{i_2 \in [0,...,n_2-1]}D_{ps}(p^1_{i_1},s^2_{i_2}).$
\end{definition}

\begin{figure}[!ht]
\centering
\includegraphics[width=\linewidth]{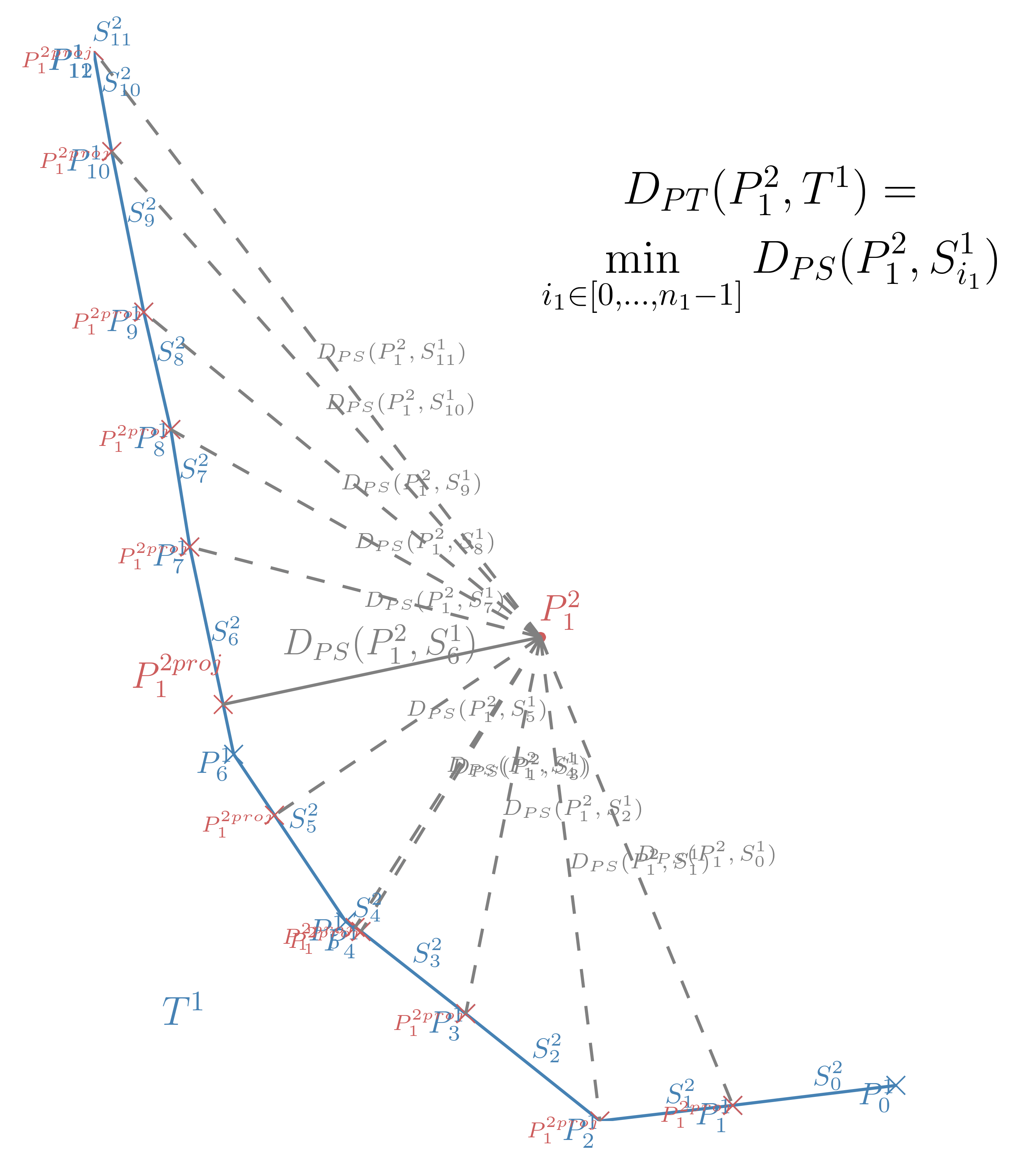}\label{figure_distance_point_trajectoire}
\caption{Distance from point $p^2_1$ to trajectory $T^1$}
\end{figure}

\begin{figure}[!ht]
\centering
\includegraphics[width=\linewidth]{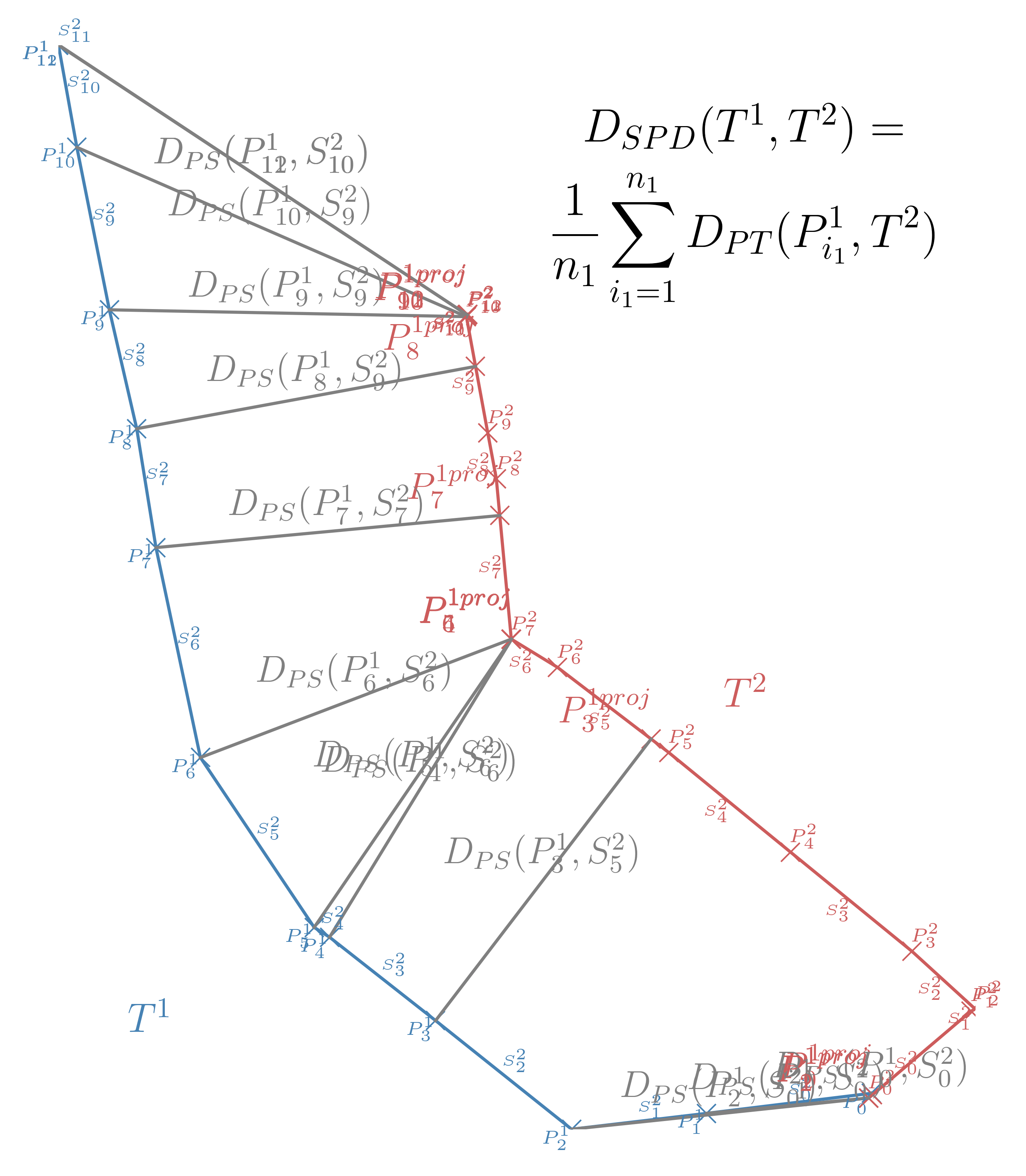}
\caption{$SPD$ Distance from trajectory $T^1$ to trajectory $T^2$ }
\label{figure_distance_trajectoire_trajectoire}
\end{figure}

\begin{proposition}\label{theorem_subtraj} If $T^1$ is a sub trajectory of $T_{pl}^2$,  $D_{SPD}(T^1,T^2)=0$.
\end{proposition}
\begin{proof}
If $T^1$ is a sub trajectory of $T_{pl}^2$, all points of $T^1$ lie within segments $s^2$ that compose $T_{pl}^2$. By definition $D_{ps}(p^1_{i_1},s^2_{i_2})=0$ $\forall p1_{i_1}\in T^1,s^2_{i_1}\in T^2_{pl}$. It follows that $D_{pt}(p^1_{i-1},T^2)=0$ $\forall p1_{i_1}\in T^1$ and finally $D_{SPD}(T^1,T^2)=0$ 
\end{proof}
This distance is not symmetric. If $T^1$ is a very small sub-trajectory of $T^2$, $D_{SPD}(T^1,T^2)=0$, $D_{SPD}(T^2,T^1)$ can be very large. By taking the mean of these distances, the {\bf "Symmetrized Segment-Path Distance"}, SSPD, is defined and is symmetric.

\begin{definition} \textit{Symmetrized Segment-Path Distance} distance \label{definition-symetrized-distance}
$$ \label{equation_distance_modified_OWD}
D_{SSPD}(T^1,T^2)=\frac{D_{SPD}(T^1,T^2)+D_{SPD}(T^2,T^1)}{2}.
$$
\end{definition}

In definitions \ref{definition-distance} and \ref{definition-symetrized-distance}, distances \textit{SPD} and \textit{SSPD} are computed by taking the mean of the \textit{Point-to-Trajectory} distance and the \textit{SPD} distance. If the maximum is used instead of the mean, one recovers the Hausdorff function between two trajectories. Computing only one distance between two locations makes it very sensitive to noise. Yet our method computes the mean of such quantities which makes it less sensitive to this noise. For example, for the trajectories in Fig. \ref{hausdorff_frechet_inconvenient}, the \textit{SSPD} distance between $T^1$ and $T^2$ is lower than the distance between $T^1$ and $T^3$ or $T^2$ and $T^3$ ($D(T^1,T^2) = 0.58,D(T^1,T^3) =  1.5,D(T^2,T^3) =  2.03$).

\begin{proposition} \textit{SSDP} is a \textit{symmetric}.
\end{proposition}
\begin{proof}
\textit{SSDP} is a sum of Euclidean distances. By definition \textit{SSDP} is greater or equal to $0$. By definition \ref{definition-symetrized-distance}, SSDP is symmetric. Finally theorem \ref{theorem_subtraj} says that, if $D_{SDP}(T^1,T^2)=0$, $T^1$ is a sub trajectory of $T^2$. Therefore if $D_{SSDP}(T^1,T^2)=0$, both $D_{SDP}(T^1,T^2)=0$ and $D_{SDP}(T^1,T^2)=0$, and $T^1=T^2$. \textit{SSDP} is then a \textit{symmetric}. 
\end{proof}

\textit{SSDP} is quite similar to \textit{OWD} but its definition resolves most of the problems of \textit{OWD} regarding the desired properties defined in \ref{desired_properties}
\begin{itemize}
\item The points coming from the interpolation of two observed locations of a trajectory are less trustworthy that the real observations. Hence, it is natural to strengthen the importance of the observed points.
\item \textit{SSPD} distance does not require any additional parameters such as a threshold or a grid to be computed.
\item Its computation cost is $O(n^2)$. It only depends on the number of locations.
\end{itemize}

\section{Clustering}\label{section_clustering}

To evaluate these different distances, we will study different clustering obtained with the same algorithm but with distances computed using all previous distances. The different selected clustering methods and the quality of cluster criterion are exposed in this section.

\subsection{Methods}
 The choice of the clustering method is restricted by the characteristics of the trajectory object. Indeed, trajectories have different lengths which prevents an easy definition of a mean trajectory object. The \textit{k-means} method cannot be used on our trajectory set, nor \textit{spectral clustering} methods. \textit{k-medoid} can be used but an efficient algorithm, like \textit{partitioning around medoids}, or \textit{dbscan} method, require a valid \textit{metrics}. Indeed, these algorithms are based on nearest neighbor and require the distance used to satisfy the triangular inequality. Most of the studied distances, \textit{SSPD}, \textit{LCSS}, \textit{DTW}, are not metrics. In this way, \textit{dbscan} or \textit{partitioning around medoids} algorithms will not be used. Moreover, \textit{dbscan} depends on two extra parameters that are hard to estimate in this case.

To perform the clustering of the trajectories, we will focus on two methodologies : \textit{hierarchical cluster analysis } (HCA) and \textit{affinity propagation} (AP). As a matter of fact, \textit{HCA} and \textit{AP} can use distance/similarity which does not satisfy the triangle inequality.  We point out that the choice of the clustering method is restricted to the trajectory object we deal with. Actually, trajectories have different lengths. \textit{HCA} and \textit{AP} are both methods which only require the distance/similarity matrix, and thus can cluster objects of different lengths. Both these methods will be used to evaluate our distance.

\subsection{Quality criterion of cluster result}
A clustering algorithm aims at gathering objects into homogeneous groups that are far one from another. Hence, the optimal number of cluster is usually selecting by looking  at the between and within variance of the obtained clusters. In this particular case, they can not be computed here because of the impossibility to compute the mean of the trajectory object. Yet, we approximate this mean by considering an exemplar of a set of a trajectory $\mathcal{T}$ of length $n^{\mathcal{T}}$, defined as 
$T^{ex}_{\mathcal{T}} = \min_{\substack{T^i\\
                i\in[0\hdots n^{\mathcal{T}}]\\
                }} \Big\{ \displaystyle \sum_{\substack{j=1\\
                j\neq i\\
                }}^{n^{\mathcal{T}}} D(T^i,T^j) \Big \}. $

Let $\mathcal{C}_1,\hdots, \mathcal{C}_K$ be a set of clusters of $\mathcal{T}$. Hence, the between and within variance are replace by the \textit{Between-Like} and the \textit{variance-like}.

\begin{definition} \textit{Between-Like} and \textit{Within-Like}
$$BC = \displaystyle \sum_{k=1}^{K} D(T^{ex}_{\mathcal{T}},T^{ex}_{\mathcal{C}_k}), $$ 
$$WC = \displaystyle \sum_{k=1}^{K} \frac{1}{|\mathcal{C}_k|} \sum_{T^i\in{C_k}} D(T^{ex}_{\mathcal{C}_k},T^i).  $$
\end{definition}

The Within-Like criterion shows the spread of elements belonging to the same cluster while the Between-Like criterion shows the spread between clusters.
As for the variance, for a given number of clusters, we want the Within-Like criterion to be as small as possible, and the Between-Like criterion to be as big as possible.

\section{Experimental evaluation}
\label{section_results}

In this section, we evaluate and compare $6$ distances  \textit{LCSS}, \textit{DTW}, \textit{Hausdorff}, \textit{Frechet}, \textit{Frechet Approximation}, and the \textit{SSDP}. All this distance have been implemented in both python and cython and are available in the \textit{traj-dist} package (https://github.com/bguillouet/traj-dist).

We also use python for the implementation of the chosen clustering algorithms, the \textit{sklearn} library for \textit{affinity propagation} and \textit{scipy} library for \textit{hierarchical clustering analysis}. For the latter, \textit{weighted}, \textit{average}, \textit{ward} and \textit{single} linkage criteria have been compared.

\subsection{The Data}

The data we used are GPS data from $536$ San-Francisco taxis over a 24-day period. These data are public and can be found in \cite{cabspotting}. We extracted a subset of this set as shown Fig. \ref{data_caltrain_downtown}.

\begin{figure}[!t]
\centering
\includegraphics[width=\linewidth]{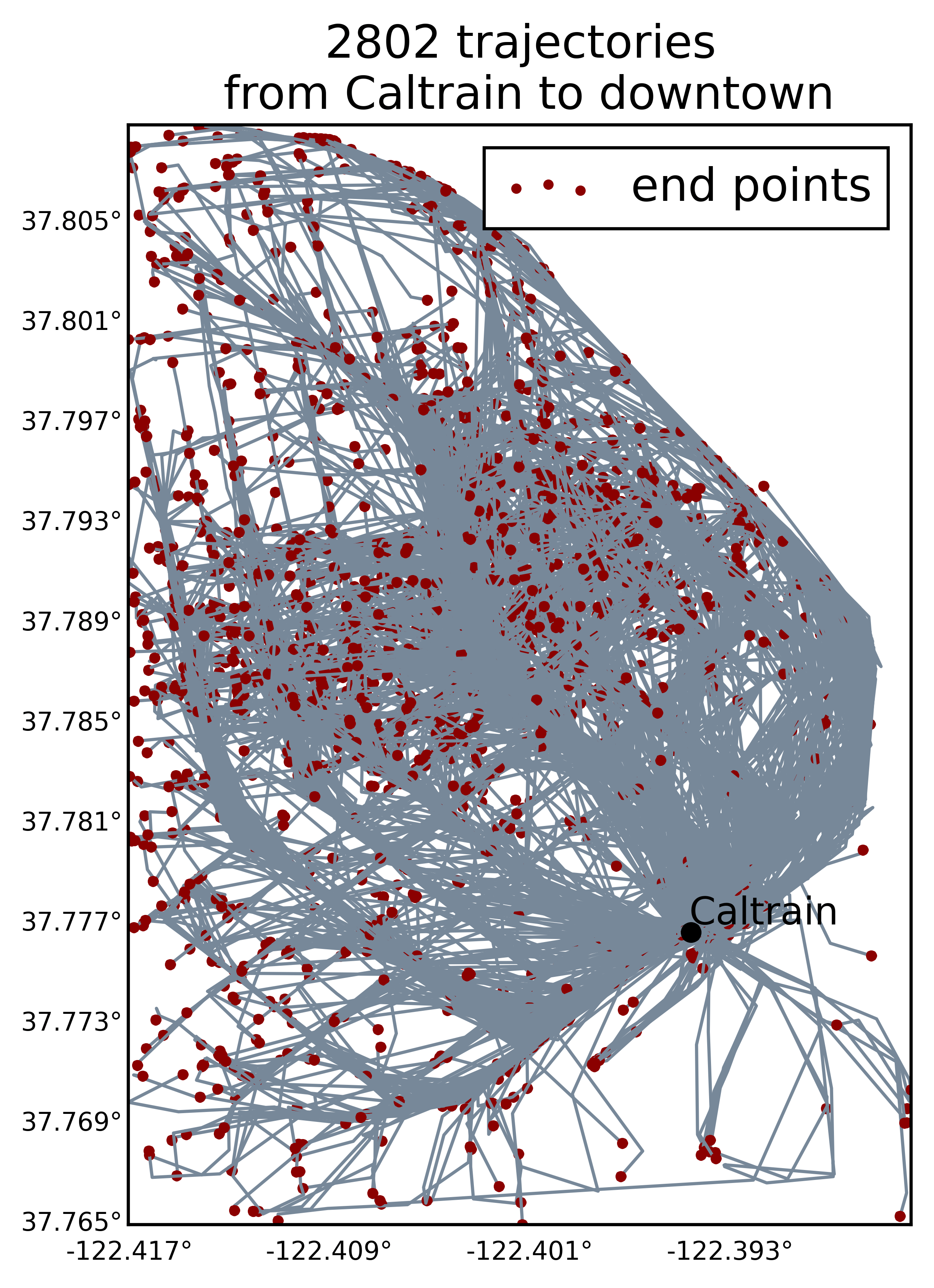}
\caption{Trajectories subset}
\label{data_caltrain_downtown}
\end{figure}

This subset is a blend of 2802 trajectories. They all have the same pickup location, the Caltrain station, and all have their drop-off location in downtown San-Francisco.

\subsection{Computation cost}
In Table \ref{table-computation} we can observe the computation time needed to compute the matrix distance for $100$ trajectories composed of between 3 and 36 locations, most having around 10.

\begin{table}[!t]
\centering
\caption{Computation Time in seconds}\label{table-computation}
\begin{tabular}{|c|c|c|} 
\hline 
Distance & Python & Cython \\\hline
Fr\'echet & 131.76 & 36.32  \\\hline
Discrete Fr\'echet & 3.67  &  2.24 \\\hline
Hausdorff & 13.36 & 0.28 \\\hline
DTW &  3.63 &  0.40 \\ \hline
LCSS &  2.79 & 0.60 \\ \hline
SSPD & 13.20 & 0.32\\ \hline
\end{tabular} 
\end{table}

Fr\'echet distance is the distance that takes most computation time. It is the only method that runs in $O(n^2log(n^2))$. With python, \textit{DTW}, \textit{LCSS} and \textit{Discrete Fr\'echet} distances are the fastest methods, while \textit{Hausdorff} and \textit{SSPD} are the fastest with cython because of its ability to declare static variables and to use the C math library. \textit{DTW}, \textit{LCSS} and \textit{Discrete Fr\'echet} each have a backtracking step which is not improved with the cython implementation. This explains the faster computing time for \textit{Hausdorff} and \textit{SSPD}.

\subsection{Analysis of the number of cluster selection}

In Fig. \ref{smowd_euclidean__medoid_variance_pct} we can observe the evolution of the within- and the between- like criterion described section \ref{section_clustering} for the distance \textit{SSPD} and for the selected methods \textit{AP} and \textit{CAH}. Both the Between-Like and the Within-Like criterion are displayed because the sum of these two criteria is not constant as opposed to the sum of the between and within variance.

\begin{figure}[!ht]
\centering
\includegraphics[width=\linewidth]{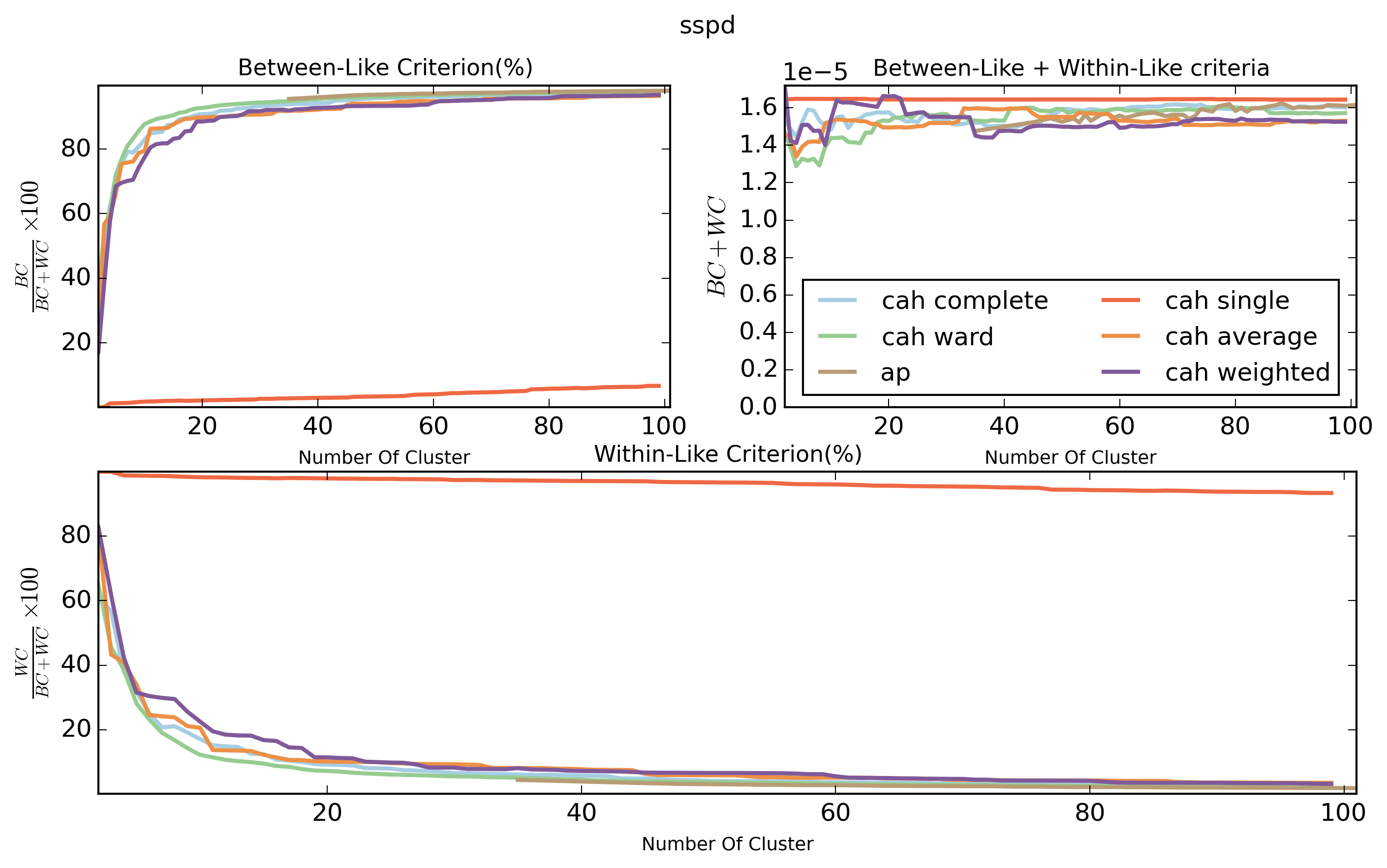}
\caption{Evolution of the Within-Like and Between-Like criteria depending on cluster.}
\label{smowd_euclidean__medoid_variance_pct}
\end{figure}

The \textit{CAH single} method gives poor results. All other methods have the same evolution of the studied criterion depending on the cluster size. A plateau can be observed starting from a clusters size between $15$ and $20$. Adding more cluster does not decrease significantly the Within-Like Criterion. Twenty is a good cluster size for the \textit{CAH method}.

\textit{CAH Ward} and \textit{AP} give the best results. But the latter does not find any clustering with less than $38$ clusters which is a too large cluster size.

The same conclusions can be made with the six studied distances.

The \textit{CAH Ward} method with cluster size of $20$ and the \textit{AP} method with the \textit{preference} parameter fixed to the minimum of the computed matrix distance will be used to compare the studied distances in more details.

\subsection{Analysis of the distances}

We can observe the evolution of the Within-Like and the Between-Like criteria for the two selected clustering methods as well as for all studied distances. The \textit{CAH WARD} results are display in Fig. \ref{compare_ward_variance}, and the \textit{AP} results in Fig. \ref{compare_ap_variance}.

\begin{figure}[!ht]
\centering
\includegraphics[width=\linewidth]{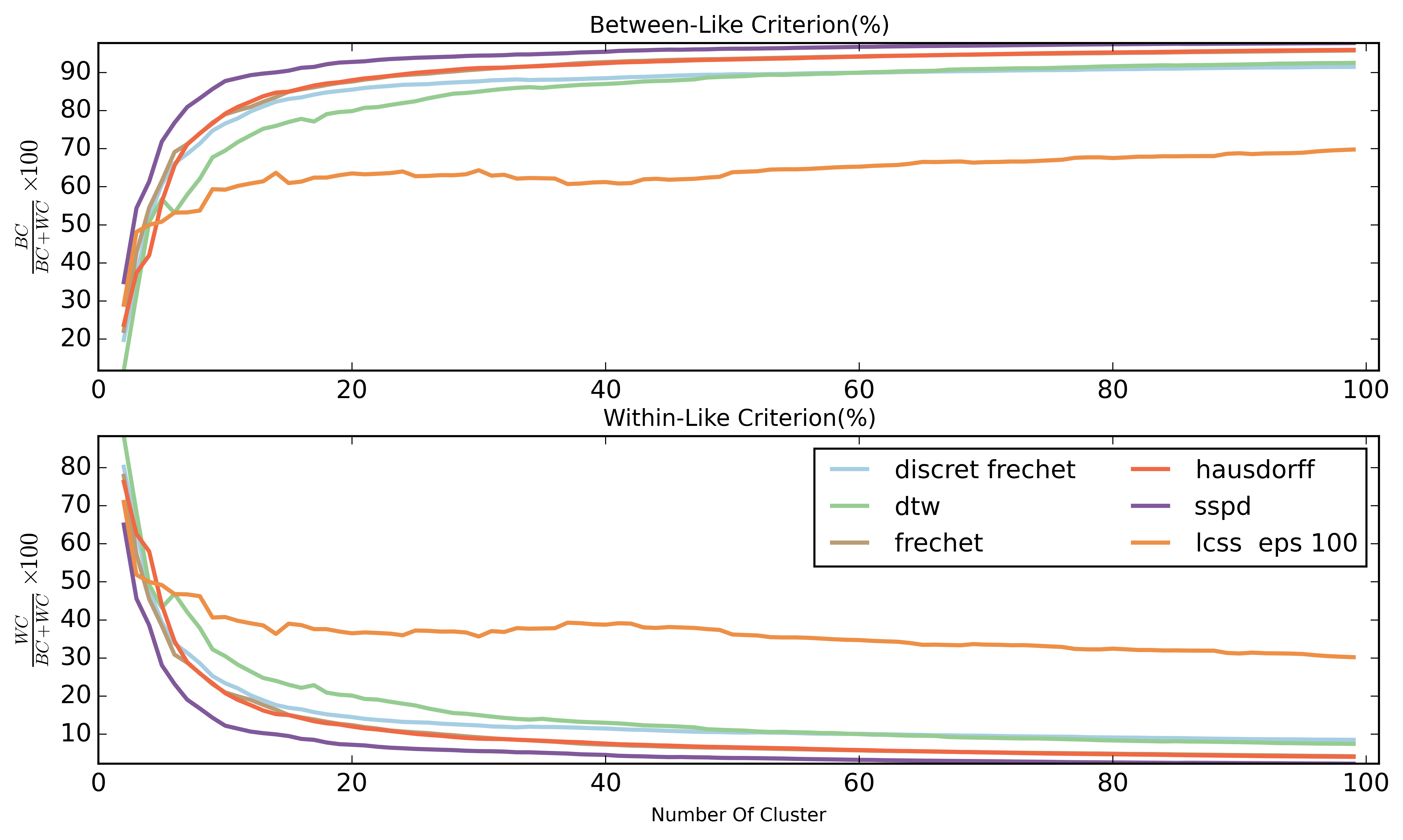}
\caption{Evolution of the Within-Like and Between-Like criteria depending on cluster size for all distances using the CAH-WARD method}
\label{compare_ward_variance}
\end{figure}

\begin{figure}[!ht]
\centering
\includegraphics[width=\linewidth]{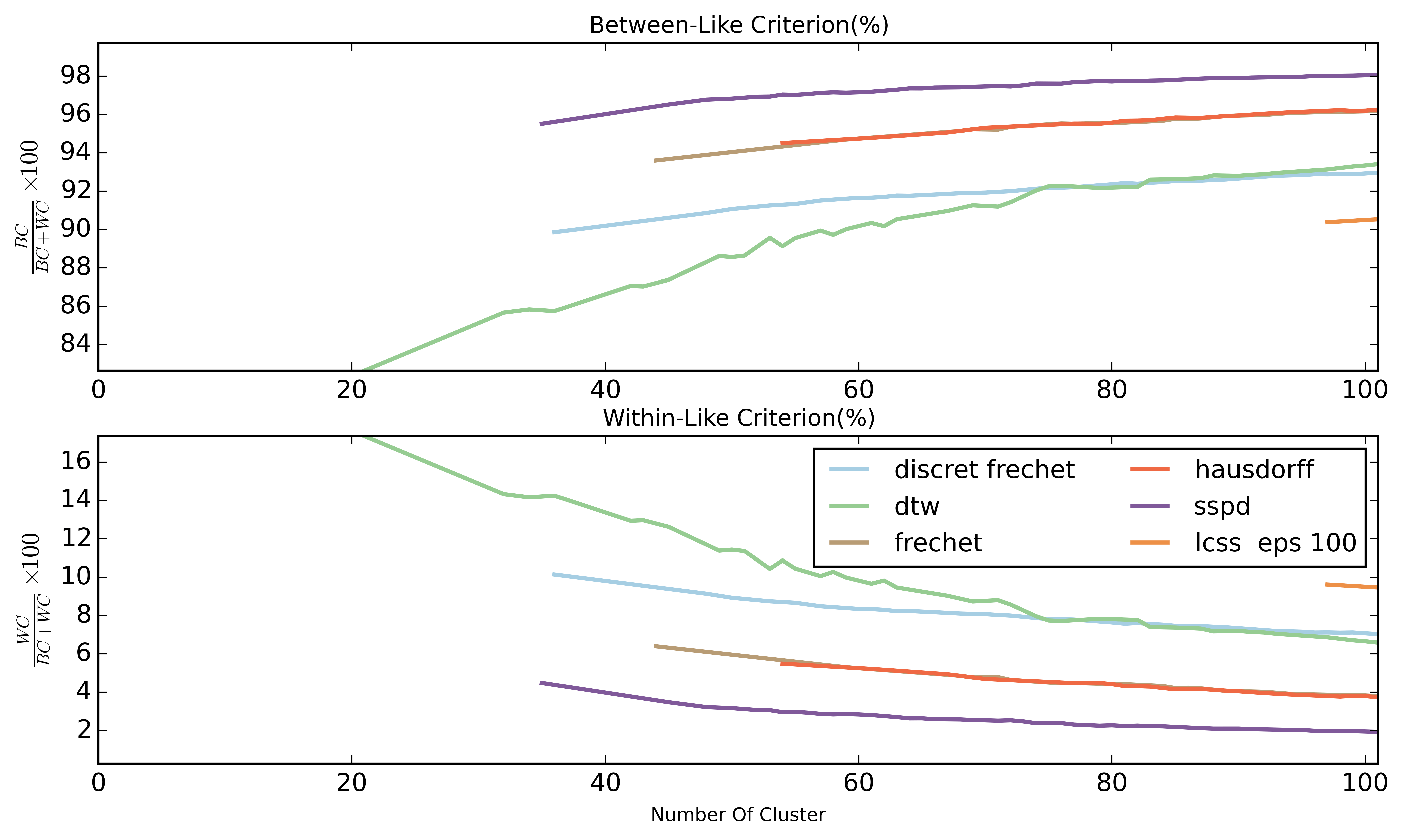}
\caption{Evolution of the Within-Like and Between-Like criteria depending on the cluster size for all distances using the AP method}
\label{compare_ap_variance}
\end{figure}

The minimum of cluster size found by the \textit{AP} method differs significantly according to the used distance. No more thant $21$ clusters are found with the \textit{DTW} distance, and $35$ with \textit{SSPD} or $54$ with \textit{Hausdorff}. 

The Warping-based distances, \textit{LCSS} and \textit{DTW}, give the poorest results with \textit{LCSS} being significantly worse than \textit{DTW}. The two shape-based distances \textit{Frechet} and \textit{Hausdorff} give better results. The evolution of their criteria is very similar to each other. The \textit{Discrete Fr\'echet} distance is between these two types of distances. These results confirm that shape-based distances are better adapted than warping-based distances for our objectives. 

Finally, the new distance \textit{SSDP} gives the best results. It has the lowest value of Within-Like Criterion for all cluster sizes and with both \textit{CAH WARD} and \textit{AP} clustering methods.

We can observe the visual results for this distance and both clustering methods, in Fig. \ref{fig:1}, and the isolated clusters, in Fig. \ref{fig2:1},

\begin{figure}[!ht]
\begin{minipage}{0.49\linewidth}
\includegraphics[width=\textwidth]{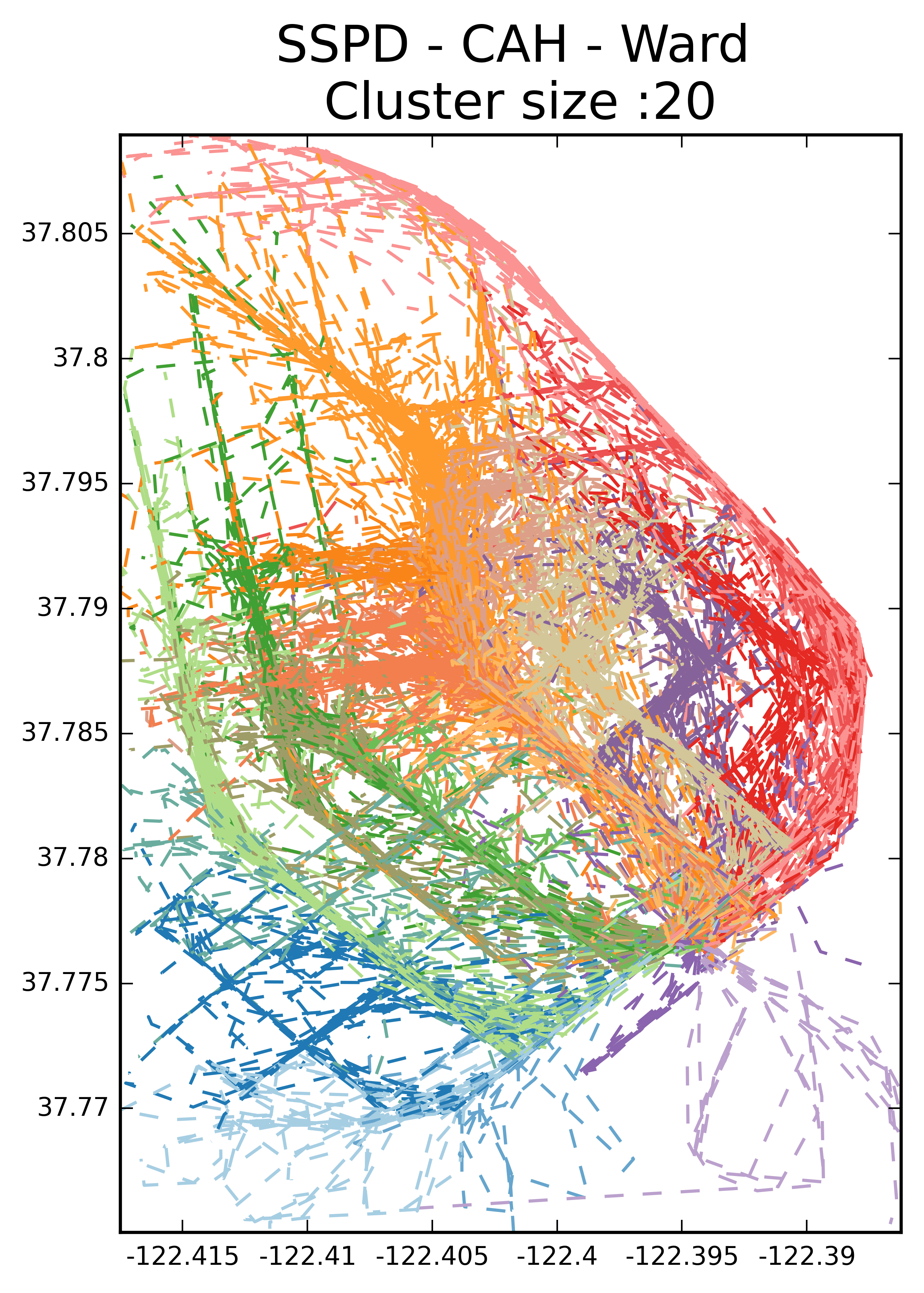}
\end{minipage}
\begin{minipage}{0.49\linewidth}
\includegraphics[width=\textwidth]{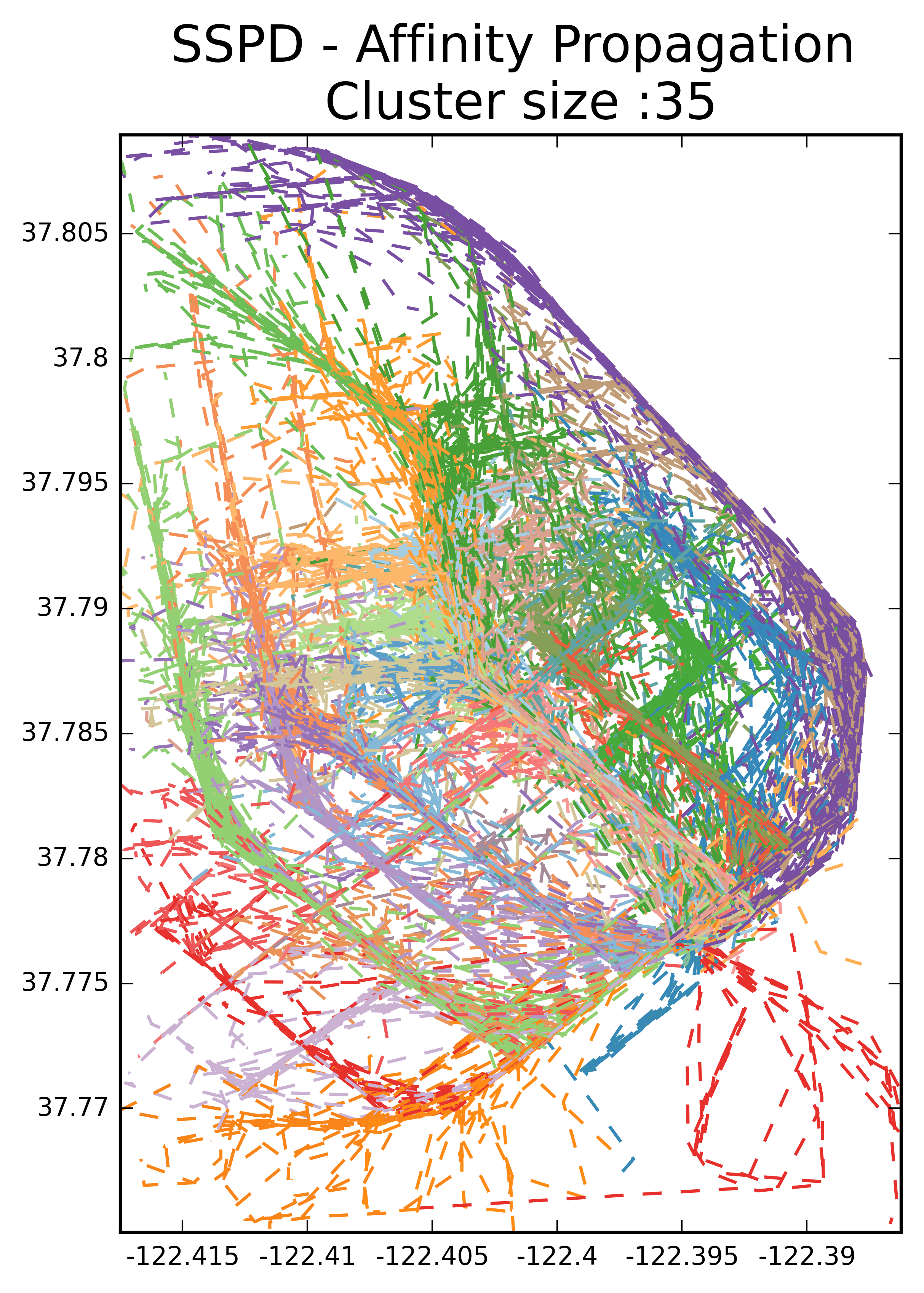}
\end{minipage}
\caption{Clustering results with \textit{SSPD} distance}
\label{fig:1}
\end{figure}

We can observe that trajectories are well classified according to their path. In Fig. \ref{fig2:1}, clusters found with \textit{CAH WARD} seems to be consistent. The cluster size with \textit{AP} method is $38$. This is a large number according to the whithin-Like criterion computed with \textit{CAH}. In fact, the Within-Like criterion does not decrease much between $20$ and $35$. However, we can see that clusters found with \textit{AP} are still consistent.  

\begin{figure}[!ht]
\begin{minipage}{0.45\linewidth}
\includegraphics[width=\textwidth]{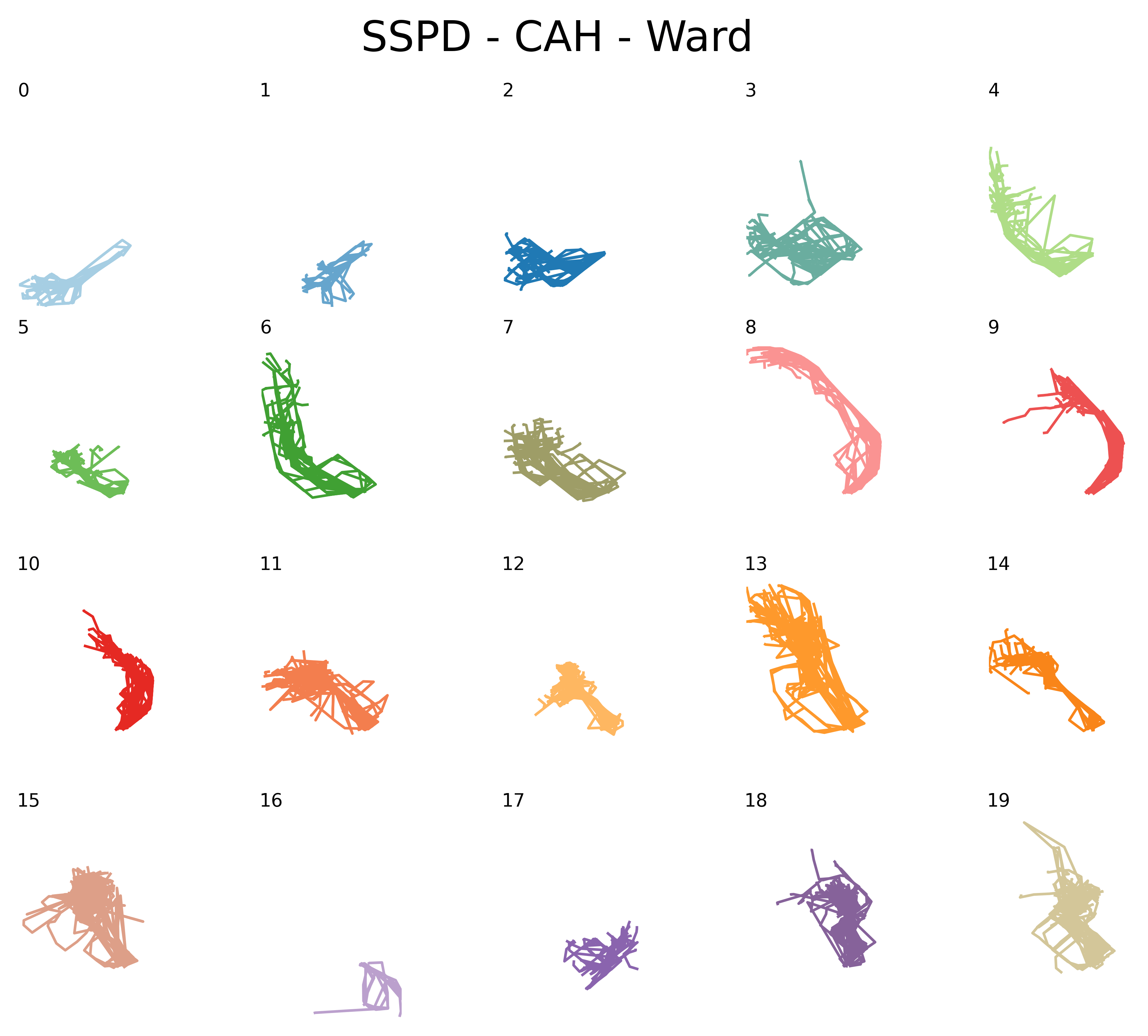}
\end{minipage}
\hspace{1cm}
\begin{minipage}{0.45\linewidth}
\includegraphics[width=\textwidth]{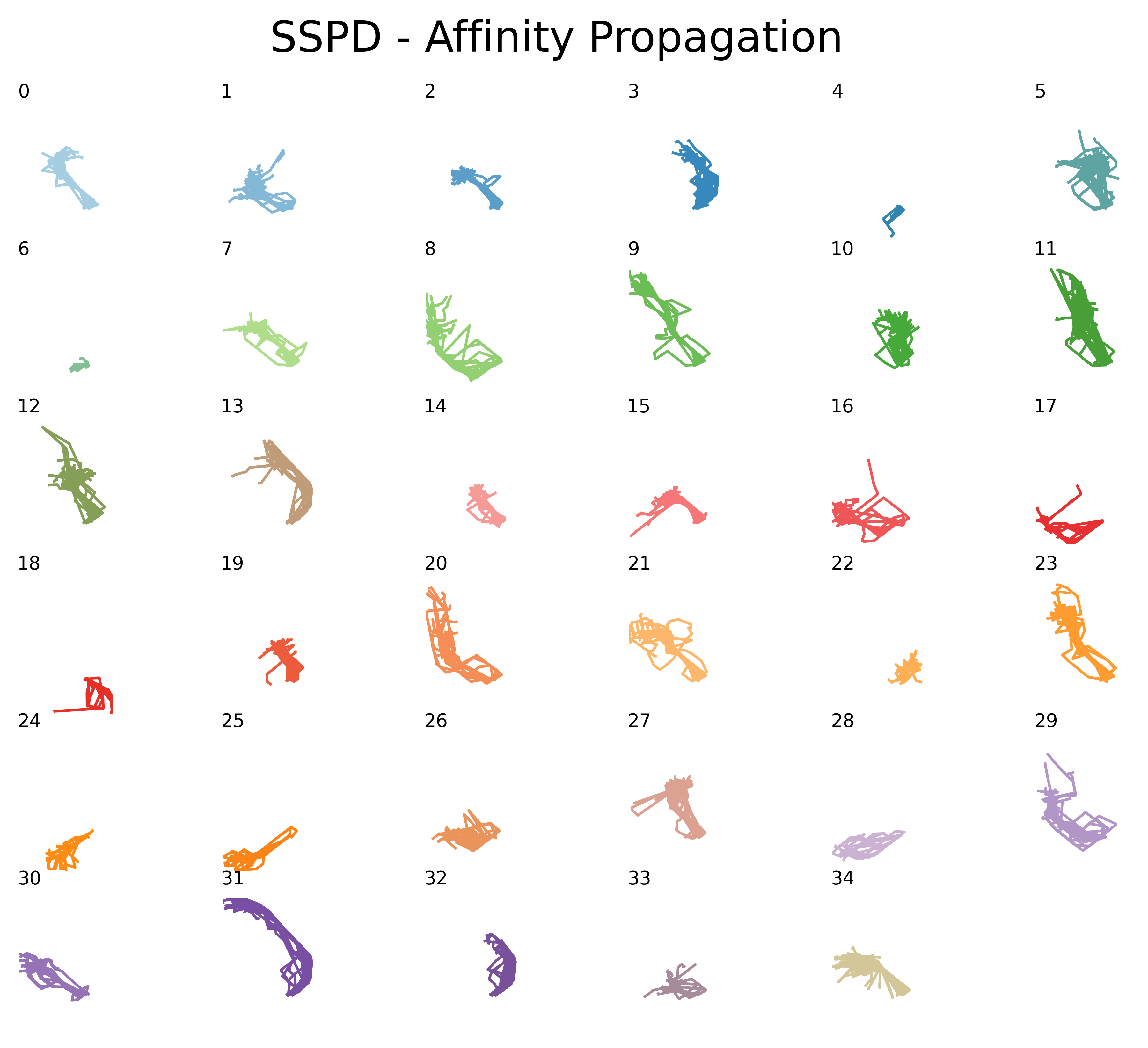}
\end{minipage}
\caption{The isolated clusters}
\label{fig2:1}
\end{figure}

A cluster computed with the \textit{CAH WARD} method based on a matrix distance computed with \textit{SSPD} gives best result. The Between-Like and Within-Like criteria show that this method is good to regroup cluster around exemplar.

\section*{Conclusion}
Clustering of non Euclidean objects deeply relies of the choice of a proper distance. For trajectories analysis, we presented different distances focusing on different features of such objects. To cope with their different weakness we propose a new distance, the \textit{Symmetrized Segment-Path Distance}. This distance is time insensitive, and compares the shape and the physical distance between two trajectory objects. It enables to obtain a good clustering using either \textit{hierarchical clustering} and \textit{affinity propagation} methods. Hence the clusters obtained are homogeneous with regard to shape and seem to properly capture the behaviours of the drivers. We have thus obtained a partition of the network based on the uses of the drivers that can still be interpreted as vehicles trajectories. Using such  features to forecast the final destination of the drivers will be tackled in a following work.


%
\nocite{*}
\bibliographystyle{IEEEtran}
\bibliography{bibl}

\end{document}